\documentclass{article}

\usepackage{arxiv}

\usepackage[utf8]{inputenc} %
\usepackage[T1]{fontenc}    %
\usepackage{hyperref}       %
\usepackage{url}            %
\usepackage{booktabs}       %
\usepackage{amsfonts}       %
\usepackage{nicefrac}       %
\usepackage{microtype}      %
\usepackage{lipsum}         %

\newcommand{\argmin}{\mathop{\rm argmin}}

\newcommand{\dquote}[1]{``#1''}

\usepackage{url}            %
\usepackage{booktabs}       %
\usepackage{multirow}    
\usepackage{amsfonts}       %
\usepackage{nicefrac}       %
\usepackage{microtype}      %
\usepackage{natbib}
\usepackage{enumerate}
\usepackage{hhline}
\usepackage{makecell}
\usepackage{pifont}

\usepackage{graphicx} %
\usepackage{caption}
\usepackage{subcaption}
\usepackage{amsmath}
\usepackage{amsthm}
\usepackage{amssymb}
\usepackage{tikz}
\usepackage{xcolor}
\usetikzlibrary{arrows}

\allowdisplaybreaks

\usepackage{mathrsfs}

\usepackage{algorithm}
\usepackage{algpseudocode}
\usepackage{hyperref}
\usepackage{bm}

\allowdisplaybreaks

\newcommand{\labs}{\left\vert}
\newcommand{\rabs}{\right\vert}

\newcommand{\real}{\mathbb{R}}

\newcommand{\expect}{\mathbb{E}}

\newtheorem{thm}{Theorem}

\newtheorem{rem}{Remark}

\newtheorem{obs}{Observation}

\usepackage{cleveref}
\crefname{thm}{Theorem}{Theorems}
\crefname{lem}{Lemma}{Lemmas}
\crefname{cor}{Corollary}{Corollaries}
\crefname{prop}{Proposition}{Propositions}
\crefname{asmp}{Assumption}{Assumptions}
\crefname{defn}{Definition}{Definitions}
\crefname{oracle}{Oracle}{Oracles}
\crefname{fact}{Fact}{Facts}
\crefname{conj}{Conjecture}{Conjectures}
\crefname{rem}{Remark}{Remarks}
\crefname{example}{Example}{Examples}
\crefname{condition}{Condition}{Conditions}
\crefname{exercise}{Exercise}{Exercises}
\crefname{algorithm}{Algorithm}{Algorithms}
\crefname{table}{Table}{Tables}
\crefname{figure}{Figure}{Figures}
\crefname{section}{Section}{Sections}
\crefname{subsection}{Section}{Sections}
\crefname{appendix}{Appendix}{Appendices}
\crefname{message}{Message}{Messages}
\crefname{obs}{Observation}{Observations}

\definecolor{red}{rgb}{1, 0, 0}

\definecolor{green}{rgb}{0, 1, 0}

\definecolor{blue}{rgb}{0, 0, 1}

\definecolor{orange}{rgb}{1, 0.4, 0.0}

\input{packages/math_commands}

\title{Rethinking ValueDice: \\ Does It Really Improve Performance?}

\newcommand{\expert}{\operatorname{exp}}
\newcommand{\piE}{\pi^{\expert}}
\newcommand{\dice}{\operatorname{DICE}}
\newcommand{\mix}{\operatorname{mix}}

\newcommand{\dpi}{d^{\pi}}
\newcommand{\dexp}{d^{\operatorname{exp}}}
\newcommand{\dmix}{d^{\mix}}
\newcommand{\drb}{d^{\operatorname{RB}}}

\newcommand{\offline}{\normalfont \texttt{offline}}

\hypersetup{
    colorlinks=true,
    citecolor=blue,
    linkcolor=blue,
}

\author{Ziniu Li\thanks{Equal contribution.} \\
	Shenzhen Research Institute of Big Data \\ The Chinese University of Hong Kong, Shenzhen \\
	\texttt{ziniuli@link.cuhk.edu.cn} \\
	\And
	Tian Xu$^{*}$ \\ 
    National Key Laboratory for Novel Software Technology \\
    Nanjing University \\
	\texttt{xut@lamda.nju.edu.cn} \\
	\And 
	Yang Yu\thanks{Corresponding author.} \\ 
    National Key Laboratory for Novel Software Technology \\
    Nanjing University \\
	\texttt{yuy@nju.edu.cn} \\
	\And
	Zhi-Quan Luo$^{\dag}$ \\
	Shenzhen Research Institute of Big Data \\ The Chinese University of Hong Kong, Shenzhen \\
	\texttt{luozq@cuhk.edu.cn} \\
}

\renewcommand{\headeright}{}
\renewcommand{\undertitle}{}
\renewcommand{\shorttitle}{Rethinking ValueDice: Does It Really Improve Performance?}

\begin{document}
\maketitle

\setcounter{footnote}{0}

\begin{abstract}
Since the introduction of GAIL, adversarial imitation learning (AIL) methods attract lots of research interests. Among these methods, ValueDice has achieved significant improvements: it beats the classical approach Behavioral Cloning (BC) under the offline setting, and it requires fewer interactions than GAIL under the online setting. Are these improvements benefited from more advanced algorithm designs? We answer this question by the following conclusions.

First, we show that ValueDice could reduce to BC under the offline setting. Second, we verify that overfitting exists and regularization matters in the low-data regime. Specifically, we demonstrate that with weight decay, BC also nearly matches the expert performance as ValueDice does. The first two claims explain the superior offline performance of ValueDice. Third, we establish that ValueDice does not work when the expert trajectory is subsampled. Instead, the mentioned success of ValueDice holds when the expert trajectory is complete, in which ValueDice is closely related to BC that performs well as mentioned. Finally, we discuss the implications of our research for imitation learning studies beyond ValueDice\footnote{This paper is presented at the blog track of the 10th international conference on learning representations (ICLR), 2022. Link: \url{https://iclr-blog-track.github.io/2022/03/25/rethinking-valuedice/}. }.

\end{abstract}

\section{Introduction}

Many practical applications involve sequential decision-making. For these applications, an agent implements a policy to select actions and maximize the long-term return. Imitation learning approaches obtain the optimal policy from expert demonstrations \citep{argall2009survey, hussein2017survey, osa2018survey}. Imitation learning has been successfully applied in game~\citep{ross11dagger, silver2016mastering}, recommendation system~\citep{shi2019taobao, chen2019recomendation}, and robotics~\citep{levin16_end_to_end, finn16guided_cost_learning}, etc.

One of the milestones in imitation learning is the introduction of generative adversarial imitation learning (GAIL) \citep{ho2016gail}. Different from the classical approach Behavioral Cloning (BC) \citep{Pomerleau91bc} that trains a policy via \emph{supervised learning}, GAIL performs \emph{state-action distribution matching} in an adversarial manner \citep{ghasemipour2019divergence, ke19imitation_learning_as_f_divergence, xu2020error}. Even when expert demonstrations are scarce, GAIL is empirically shown to match the expert performance while BC cannot \citep{ho2016gail}. However, the price is that GAIL requires a number of environment interactions (e.g., 3M for Gym MuJoCo locomotion tasks), which restricts GAIL to be applied under the online setting. In contrast, BC does not require any interaction thus BC is widely applied under the offline setting. 

\textbf{Context.} The remarkable achievements of GAIL attract lots of research interests in adversarial imitation learning (AIL) approaches. Among these methods, ValueDice \citep{Kostrikov20value_dice} has achieved significant improvements. On the one hand, ValueDice beats BC under the offline setting. On the other hand, ValueDice outperforms GAIL and another off-policy AIL algorithm DAC \citep{Kostrikov19dac} in terms of interaction efficiency. All existing results suggest ValueDice is perfect. This motivates the central question in this manuscript: are these improvements benefited from more advanced algorithm designs?

Before answering the raised question, we have to carefully examine the achievements of ValueDice.  For the comparison with BC, the empirical advance of ValueDice seems to contradict the recent lower bound that BC is (nearly) minimax optimal under the offline setting \citep{rajaraman2020fundamental, xu2021error}. In other words, this theoretical result suggests that no algorithm can beat BC \emph{in the worst case} under the offline setting. For the online part, it is well accepted that optimization of Bellman backup objectives (e.g., temporal-difference learning and Q-learning) with function approximation can easily diverge (see the divergence example in \citep{baird1995residual, tsitsiklis1997analysis}). To address this divergence issue, \emph{target network} is proposed in \citep{mnih2015human} and this technique is widely applied in deep RL \citep{lillicrap2015ddpg, fujimoto2018td3, haarnoja2018sac}; see \citep{lee2019target, zhang2021breaking, agarwal2021online, chen2022target, li2022note} for the explanation about why the target network can address the divergence issue. However, unlike DAC that uses the target network, ValueDice is shown to successfully perform off-policy optimization \emph{without} the target network. We will provide explanations for these \dquote{contradictions} when we answer the raised question.

\textbf{Main Contributions.} First, we show that ValueDice could reduce to BC under the offline setting. Second, we verify that \emph{overfitting} exists and \emph{regularization} matters in the low-data regime. Specifically, we demonstrate that with weight decay, BC also nearly matches the expert performance as ValueDice does. Actually, ValueDice has implemented the orthogonal regularization \citep{brock2018large}.  The first two claims explain the superior offline performance of ValueDice. Third, we establish that ValueDice \emph{cannot} successfully optimize Bellman backup objectives when the expert trajectory is subsampled. As a result, ValueDice is inferior to existing methods (e.g., DAC and GAIL) in this case. Instead, the mentioned success of ValueDice holds when the expert trajectory is complete, in which ValueDice is closely related to BC that performs well as mentioned.

For general imitation learning research, our studies have two implications. For one thing, we additionally prove that $\ell_1$-norm based state-action distribution matching (a general AIL method) is also equivalent to BC under the offline setting under mild conditions. This suggests offline AIL is expected to be minimax optimal under the offline setting as BC does. For another thing, the success of BC (with regularization) discloses limitations of existing MuJoCo locomotion benchmarks: deterministic transitions and limited initial states. For such tasks, BC has no compounding errors and performs well even provided 1 expert trajectory, which is not revealed in the previous research.

\section{Background}
\label{sec:background}

\textbf{Markov Decision Process.} Following the setup in \citep{Kostrikov20value_dice}, we consider the infinite-horizon Markov decision process (MDP) \citep{puterman2014markov}, which can be described by the tuple $\gM = (\gS, \gA, p, r, p_0, \gamma)$. Here $\gS$ and $\gA$ are the state and action space, respectively. $p_0(\cdot)$ specifies the initial state distribution while $\gamma \in [0, 1)$ is the discount factor. For a specific state-action pair $(s, a)$, $p(s^\prime|s, a)$ defines the transition probability of the next state $s^\prime$ and $r(s, a) \in [0, 1]$ assigns a reward signal. 

To interact with the environment (i.e., an MDP), a stationary policy $\pi: \gS \rar \Delta(\gA)$ is introduced, where $\Delta(\gA)$ is the set of probability distributions over the action space $\gA$. Specifically, $\pi(a|s)$ determines the probability of selecting action $a$ on state $s$. The performance of policy $\pi$ is measured by \emph{policy value} $V(\pi)$, which is defined as the cumulative and discounted rewards. That is, 
\begin{align*}
    V(\pi) := \expect_{\pi} \ls \sum_{t=0}^{\infty} \gamma^{t} r(s_t, a_t) \mid s_0 \sim p_0(\cdot), a_t \sim \pi(\cdot|s_t), s_{t+1} \sim p(\cdot|s_t, a_t) \rs.
\end{align*}
To facilitate later analysis, we need to introduce the (discounted) state-action distribution $\dpi(s, a)$:
\begin{align*}
    \dpi(s, a) = (1-\gamma) \sum_{t=0}^{\infty} \gamma^{t}\sP \lp s_t= s, a_t =a \mid s_0 \sim p_0(\cdot), a_t \sim \pi(\cdot|s_t), s_{t+1} \sim p(\cdot|s_t, a_t) \rp.
\end{align*}
With a little of notation abuse, we can define the (discounted) state distribution $\dpi(s)$:
\begin{align*}
    \dpi(s) = (1-\gamma) \sum_{t=0}^{\infty} \gamma^{t}\sP \lp s_t= s \mid s_0 \sim p_0(\cdot), a_t \sim \pi(\cdot|s_t), s_{t+1} \sim p(\cdot|s_t, a_t) \rp.
\end{align*}

\textbf{Imitation Learning.} The goal of imitation learning is to learn a high-quality policy from expert demonstrations. To this end, we often assume there is a nearly optimal expert policy $\piE$ that could interact with the environment to generate a dataset $\gD = \{ (s_i, a_i) \}_{i=1}^{m}$. Then, the learner can use the dataset $\gD$ to mimic the expert and to obtain a good policy. The quality of imitation is measured by the \emph{policy value gap} $V(\piE) - V(\pi)$ \citep{pieter04apprentice, syed07game, ross2010efficient}. The target of imitation learning is to minimize the policy value gap between the expert policy $\piE$ and the learner's policy $\pi$:
\begin{align*}
    \min_{\pi} V({\piE}) - V(\pi).
\end{align*}
Notice that this policy optimization is performed without the reward signal. In this manuscript, we mainly focus on the case where the expert policy is deterministic, which is true in many applications and is widely used in the literature.

\textbf{Behavioral Cloning.} Given the state-action pairs provided by the expert, Behavioral Cloning (BC) \citep{Pomerleau91bc} directly learns a policy mapping to minimize the policy value gap. Specifically, BC often trains a parameterized classifier or regressor with the maximum likelihood estimation:
\begin{align}   \label{eq:bc_loss}
    \min_{\theta} \sum_{(s, a) \in \gD} -\log \pi_{\theta}(a|s),
\end{align}
which can be viewed as the minimization of the KL-divergence between the expert policy and the learner's policy \citep{ke19imitation_learning_as_f_divergence, ghasemipour2019divergence, xu2020error}: 
\begin{align*}
    \min_{\pi} \expect_{s \sim \dexp(\cdot)} \ls \KL (\piE(\cdot| s) || \pi(\cdot|s) ) \rs,
\end{align*}
where $\KL(p||q) = \sum_{x \in \gX} p(x) \log (p(x)/q(x))$ for two distributions $p$ and $q$ on the common support $\gX$. 
Under the tabular setting where the state space and action space are finite, the optimal solution to \eqref{eq:bc_loss} could take in a simple form:
\begin{align}   \label{eq:bc_count}
\pi^{\operatorname{BC}}(a|s) = \left\{ \begin{array}{cc}
 \frac{ \sum_{(s^\prime, a^\prime) \in \gD}  \1(s^\prime = s, a^\prime = a)}{\sum_{(s^\prime, a^\prime) \in \gD} \1(s^\prime = s)}     & \quad  \text{if} \,\, \sum_{(s^\prime, a^\prime) \in \gD} \1(s^\prime = s) > 0 \\
  \frac{1}{|\gA|}   &  \quad \text{otherwise} 
\end{array} \right.
\end{align}
That is, BC estimates the empirical action distribution of the expert policy on visited states from the expert demonstrations.

\textbf{ValueDice.} With the state-action distribution matching, ValueDice \citep{Kostrikov20value_dice} implements the objective:
\begin{align}    
   \max_{\pi} -\KL ( \dpi || \dexp ) &:= \expect_{(s, a) \sim d^{\pi}} \ls \log \frac{\dexp (s, a)}{\dpi (s, a)} \rs  \label{eq:state_action_kl_matching_1}  \\
    &= \expect\ls \sum_{t=0}^{\infty} \gamma^{t} \log \frac{\dexp (s_t, a_t)}{\dpi (s_t, a_t)} \mid  s_0 \sim p_0(\cdot), a_t \sim \pi(\cdot|s_t), s_{t+1} \sim p(\cdot|s_t, a_t) \rs . \label{eq:state_action_kl_matching_2}
\end{align}
That is, \eqref{eq:state_action_kl_matching_2} amounts a max-return RL problem with the reward $\widetilde{r}(s, a) = \log(\dexp (s, a))/ \log (\dpi(s, a))$. In ValueDice, a dual form of \eqref{eq:state_action_kl_matching_2} is presented:
\begin{align*}
   -\KL ( \dpi || \dexp ) = \min_{x: \gS \times \gA \rar \real} \ls \log \expect_{(s, a) \sim \dexp} [e^{x(s, a)}] - \expect_{(s, a) \sim \dpi}[x(s, a)] \rs.
\end{align*}
With the trick of variable change, \citet{Kostrikov20value_dice} derived the following optimization problem: 
\begin{align} \label{eq:value_dice_objective}
    \max_{\pi} \min_{\nu: \gS \times \gA \rar \real } J_{\dice} (\pi, \nu):= \log \lp \expect_{(s, a) \sim \dexp}[e^{\nu(s, a) - \gB^\pi \nu(s, a)}] \rp - (1-\gamma) \cdot \expect_{s_0 \sim p_0(\cdot), a_0 \sim \pi(\cdot|s_0)}[\nu(s_0, a_0)],
\end{align}
where $\gB^{\pi}\nu(s, a) = \gamma \expect_{s^\prime \sim p(\cdot|s, a), a^\prime \sim \pi(\cdot |s^\prime)}\ls \nu(s^\prime, a^\prime) \rs$ performs one-step Bellman update (with zero reward). As noted in \citep{Kostrikov20value_dice}, objective \eqref{eq:value_dice_objective} only involves the samples from the expert demonstrations to update, which may lack diversity and hamper training. To address this issue, \citet{Kostrikov20value_dice} proposed to use samples from the replay buffer under the online setting. Concretely, an alternative objective is introduced:
\begin{align*}
    J_{\dice}^{\mix} &:= \log \lp \expect_{(s, a) \sim \dmix}[e^{\nu(s, a) - \gB^\pi \nu(s, a)}] \rp - (1-\alpha) (1-\gamma) \cdot \expect_{s_0 \sim p_0(\cdot), a_0 \sim \pi(\cdot|s_0)}[\nu(s_0, a_0)] \\
    &\quad - \alpha \expect_{(s, a) \sim \drb}\ls \nu(s, a) - \gB^{\pi} \nu(s, a) \rs,
\end{align*}
where $\dmix(s, a) = (1-\alpha) \dexp(s, a) + \alpha \drb(s, a)$. In this case, VauleDice is to perform the modified state-action distribution matching: 
\begin{align*}
   \max_{\pi}\quad  -\KL ( (1-\alpha) \dpi + \alpha \drb || (1-\alpha) \dexp + \alpha\drb ).
\end{align*}
As long as $\alpha \in [0, 1)$, the global optimality of $\pi = \piE$ is reserved. Under the offline setting, we should set $\alpha = 0$.

\section{Rethinking ValueDice Under the Offline Setting}

In this section, we revisit ValueDice under the offline setting, where only the expert dataset is provided and environment interaction is not allowed. In this case, the simple algorithm BC can be directly applied while on-policy adversarial imitation learning (AIL) methods cannot. This is because the latter often requires environment interactions to evaluate the state-action distribution $\dpi(s, a)$ to perform optimization.

Under the offline setting, the classical theory suggests that BC suffers the compounding errors issue \citep{ross11dagger, rajaraman2020fundamental, xu2020error}. Specifically, once the agent makes a wrong decision (i.e., an imperfect imitation due to finite samples), it may visit an unseen state and make the wrong decision again. In the worst case, such decision errors accumulate over $H$ time steps and the agent could obtain zero reward along such a bad trajectory. This explanation often coincides with the empirical observation that BC performs poorly when expert demonstrations are scarce. Thus, it is definitely interesting and valuable if some algorithms can perform better than BC under the offline setting. Unfortunately, the information-theoretical lower bound \citep{rajaraman2020fundamental, xu2021error} suggests that BC is (nearly) minimax optimal under the offline setting. 

\begin{thm}[Lower Bound \citep{xu2021error}]
Given expert dataset $\gD = \{ (s_i, a_i) \}_{i=1}^m$ which is i.i.d. drawn from $\dexp$, for any offline algorithm $\mathrm{Alg}: \gD \rightarrow \widehat{\pi}$, there exists a constant $\delta_0 \in (0, 1/10]$, an MDP $\gM$ and a deterministic expert policy $\piE$ such that, for any $\delta \in (0, \delta_0)$, with probability at least $1-\delta$, the policy value gap has a lower bound:
\begin{align*}
    V({\piE}) - V({\widehat{\pi}}) \succsim \mathrm{min} \lp \frac{1}{1-\gamma}, \frac{\vert \gS \vert}{(1-\gamma)^2 m} \rp.
\end{align*}
\end{thm}

\begin{rem}
We note that the sample complexity of BC is $\widetilde{\gO}({|\gS|}/{((1-\gamma)^2 \varepsilon)})$ \citep{xu2021error}, which matches the lower bound up to logarithmic terms. We emphasize that the minimax optimality of BC does not mean that BC is optimal for each task. Instead, it says that the worst case performance of BC is optimal. The worst case performance of offline imitation learning algorithms can be observed on the Reset Cliff MDPs in \citep{rajaraman2020fundamental, xu2021error}, which share many similarities with MuJoCo locomotion tasks\footnote{The similarity is explained as follows. Both tasks involve a bad absorbing state: once the agent takes a wrong action, it goes to this bad absorbing state and obtains a zero reward forever.}. To this end, we can believe that BC is optimal for MuJoCo locomotion tasks under the offline setting.
\end{rem}

\begin{figure}[htbp]
    \centering
    \includegraphics[width=\linewidth]{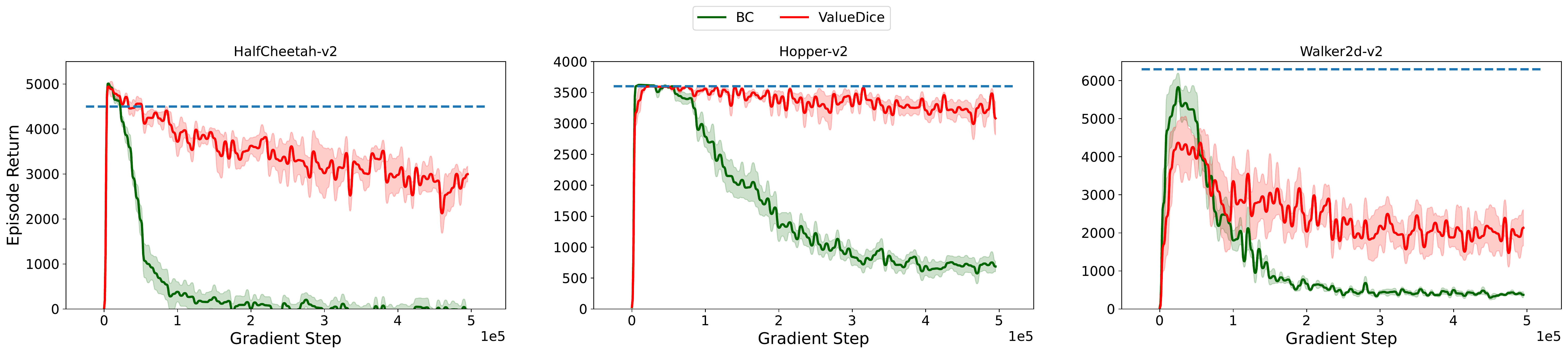}
    \caption{Comparison of ValueDice and BC under the offline setting with 1 expert trajectory, which reproduces \citep[Figure 4]{Kostrikov20value_dice}. Dashed lines correspond to the performance of expert policies.}
    \label{fig:value_dice_reproduce}
\end{figure}

From the empirical side, we realize that \citet{Kostrikov20value_dice} showed that their proposed algorithm \emph{ValueDice} could be better than BC on many interesting tasks under the offline setting; see \citep[Figure 4]{Kostrikov20value_dice} or the reproduced results\footnote{Our implementation of BC is different from the one in \citep{Kostrikov20value_dice}. In particular, our BC policy is deterministic and does not learn the covariance, while the BC policy in \citep{Kostrikov20value_dice} is stochastic. We note that although we may learn a stochastic policy, it is important for the evaluation performance to output a deterministic action (i.e., the mean) \citep{ho2016gail, haarnoja2018sac}. In fact, we observe that learning a stochastic policy by sharing the hidden layers really hurts the performance of BC.  Nevertheless, we would like to clarify that the implementation of BC in \citep{Kostrikov20value_dice} aims for a \dquote{fair} comparison because ValueDice uses a stochastic policy.} in the following \cref{fig:value_dice_reproduce} (all experiment details can be found in \cref{appendix:experiment_details}). This manifests a gap between theory and practice. In the sequel, we examine the achievement of ValueDice and explain this gap.

\subsection{Connection Between ValueDice and BC Under Offline Setting}

As a first step, let us review the training objective of ValueDice in the offline scenario. For ease of presentation, let us consider the case where the number of expert trajectory is 1, which follows the experimental setting used in \cref{fig:value_dice_reproduce}. As a result, the training objective of ValueDice with $T$ samples is formulated as 
\begin{align}  \label{eq:value_dice_empirical_surrogate_objective}
   J_{\dice}^{\offline}(\pi, \nu) &:= \log \lp \sum_{t=1}^{T} \ls e^{\nu(s_t, a_t) - \gamma \expect_{\widetilde{a}_{t+1} \sim \pi(\cdot|s_{t+1})}[ \nu(s_{t+1}, \widetilde{a}_{t+1})]} \rs \rp  - (1-\gamma) \sum_{t=1}^{T} \expect_{\widetilde{a}_t \sim \pi (\cdot|s_t) }[\nu(s_t, \widetilde{a}_t)], 
\end{align}

\begin{rem}   \label{remark:value_dice_opt}
Even derived from the principle of state-action distribution matching (i.e., objective \eqref{eq:state_action_kl_matching_1}), under the offline setting, we argue that ValueDice cannot enjoy the benefit of state-action distribution matching\footnote{The benefit is explained as follows. Under the online (or the known transition) setting, state-action distribution matching methods can perform policy optimization on non-visited states. With additional transition information, state-action distribution matching methods could have better sample complexity than BC under mild conditions \citep{xu2021nearly}.}. This is because objective \eqref{eq:value_dice_empirical_surrogate_objective} is performed only on states collected by the expert policy. That is, there is no guidance for policy optimization on non-visited states. This is the main difference between the online setting and offline setting. As a result, we cannot explain the experiment results in \cref{fig:value_dice_reproduce} by the tailored theory for online state-action distribution matching in \citep{xu2021nearly}. 
\end{rem}

To obtain some intuition about the optimal solutions to the empirical objective \eqref{eq:value_dice_empirical_surrogate_objective}, we can further consider the case where $\gamma = 0$. Note that even this variant could also outperform BC under the offline setting (see \cref{fig:value_dice_offline_variant}).  

\begin{figure}[htbp]
    \centering
    \includegraphics[width=\linewidth]{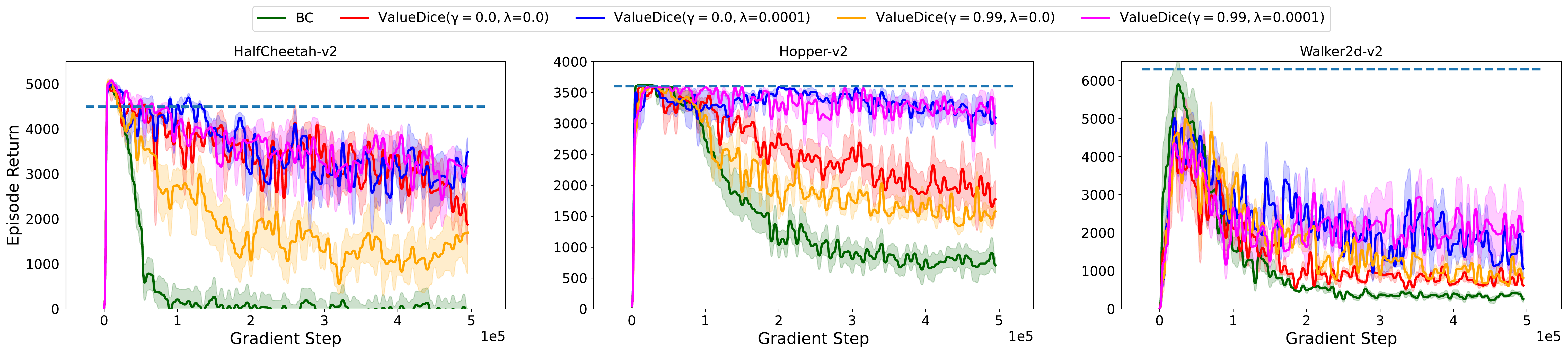}
    \caption{Comparison of BC and ValueDice's variants (with different discount factors $\gamma$ and orthogonal regularization coefficients $\lambda$) under the offline setting with 1 expert trajectory. The default configuration for ValueDice is $\gamma=0.99$ and $\lambda = 0.0001$. Dashed lines correspond to the performance of expert policies.}
    \label{fig:value_dice_offline_variant}
\end{figure}

\begin{thm}  \label{theorem:reduction_to_bc_gamma_0}
Consider the tabular MDPs. If $\gamma = 0$,  objective \eqref{eq:value_dice_empirical_surrogate_objective} reduces to 
\begin{align}   \label{eq:value_dice_empirical_surrogate_objective_gamma_0}
J_{\dice}^{\offline, \gamma=0}(\pi, \nu) &:= \log \lp \sum_{t=1}^{T} e^{\nu(s_t, a_t)} \rp -  \sum_{t=1}^{T} \expect_{\widetilde{a}_t \sim \pi (\cdot|s_t) }[ \nu(s_t, \widetilde{a}_t)].
\end{align}
Importantly, the BC's solution as in \eqref{eq:bc_count} is also one of the globally optimal solutions for ValueDice's objective in \eqref{eq:value_dice_empirical_surrogate_objective_gamma_0}.
\end{thm}

\begin{proof}
For tabular MDPs, each element is independent. As a consequence, for a specific state-action pair $(s_t, a_t)$ appeared in the expert demonstrations, its optimization objective is 
\begin{align*}
    \max_{\pi(\cdot|s_t) \in \Delta(\gA)} \quad \min_{\nu(s_t, a_t) \in \real} \nu(s_t, a_t) + b - \expect_{ \widetilde{a}_t \sim \pi (\cdot|s_t)}[\nu(s_t, \widetilde{a}_t)],
\end{align*}
which has equilibrium points when $\pi(a_t|s_t) = 1$ and $\pi(a |s_t) = 0$ for any $a \ne a_t$. To this end, the \dquote{counting} solution in \eqref{eq:bc_count} for BC is also one of the globally optimal solutions for the empirical objective \eqref{eq:value_dice_empirical_surrogate_objective_gamma_0}.
\end{proof}

\begin{figure}[htbp]
    \centering
    \includegraphics[width=\linewidth]{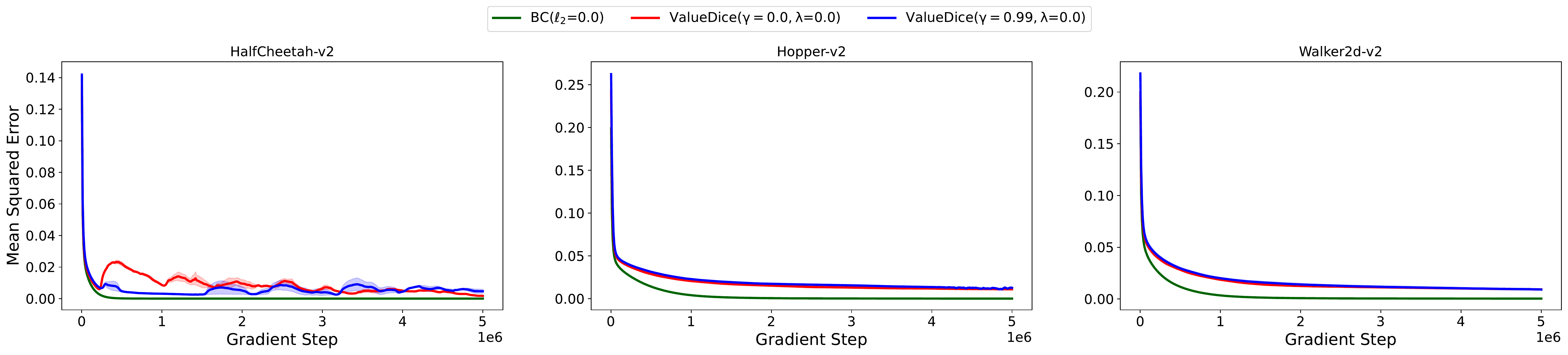}
    \caption{Comparison of BC and ValueDice in terms of MSE (i.e., BC's loss) under the offline setting with 1 expert trajectory. This result suggests that ValueDice is closely related to BC under the offline setting.}
    \label{fig:bc_loss}
\end{figure}

We remark that \cref{theorem:reduction_to_bc_gamma_0} provides a restricted case where ValueDice could degenerate to BC under the offline setting; see the experimental verification in \cref{fig:bc_loss}. As we have analyzed in \cref{remark:value_dice_opt}, the same intuition is expected to hold when $\gamma > 0$ as ValueDice does not have guidance on non-visited states. However, it is technically challenging to draw this conclusion. Instead, we consider another state-action distribution matching method with the $\ell_1$-norm metric. In particular, we claim that this conclusion still holds for tabular MDPs with finite horizons.

\begin{thm}[Informal version of \cref{theorem:episodic_mdp} in Appendix]    \label{theorem:episodic_mdp_reduction}
Consider the tabular MDPs with finite horizons. Assume that there is 1 expert trajectory provided. Under the offline setting, we have that BC's solution is the unique globally optimal solution to $\ell_1$-norm based state-action distribution matching (i.e., $\min_{\pi} \Vert \dpi - \widehat{\dexp} \Vert_1$), where $\widehat{\dexp}$ is the estimated expert state-action distribution from expert demonstrations. 
\end{thm}

See \cref{appendix:reduction_from_ail_to_bc} for the formal argument and proof. Note that the assumption of 1 expert trajectory is mainly to make the analysis simple and $\ell_1$-norm based state-action distribution matching is to ensure that the objective is well-defined\footnote{In contrast, the KL-divergence used in ValueDice is not well-defined for some state-action pair $(s, a)$ with $\dpi(s, a) = 0$ but $\dexp(s, a) > 0$.}. Actually, the experiment results in \cref{fig:value_dice_reproduce} satisfy the assumptions required in \cref{theorem:episodic_mdp_reduction}. We remark that \cref{theorem:episodic_mdp_reduction} establishes a one-to-one connection between BC and ValueDice under the \emph{offline} setting, which adds a constraint on the optimization procedure (i.e., optimization is defined over visited states). We conjecture that this conclusion holds for other AIL methods. Finally, we note that the one-to-one relationship is not disclosed by \cref{theorem:reduction_to_bc_gamma_0}.

In summary, we could believe that under the offline setting, AIL methods could not perform optimization on non-visited states. As a result, the BC policy is sufficient to be the globally optimal solution of AIL methods\footnote{Recently, {\citet{swamy2021moments} showed an interesting conclusion that BC and ValueDice can be viewed as \dquote{off-Q} algorithms. Based on this conclusion, \citet{swamy2021moments} argued that BC and ValueDICE should have similar policy value gaps. We note that this conclusion mainly holds in the \emph{population} level (informally speaking, \dquote{population} means that there are infinite samples). In contrast, we focus on the \emph{finite-samples} setting and and validate that offline AIL methods can directly reduce to BC. One important message that we want to convey is that in the offline setting with finite samples, the policy optimization only involves visited states so the fundamental limit (i.e., the compounding errors issue) cannot be overcome.}}. We highlight that this insight is not illustrated by the lower bound argument \citep{rajaraman2020fundamental, xu2021error}, which only tells us that one specific algorithm, BC, is nearly minimax optimal under the offline setting. From another viewpoint, our result also suggests that state-action distribution matching is expected to be a \dquote{universal} principle: it is closely related to BC under the offline setting while it enjoys many benefits under the online setting.

\subsection{Overfitting of ValueDice and BC}

Based on the above conclusion that AIL may reduce to BC under the offline setting, it is still unsettled that why ValueDice outperforms BC in \cref{fig:value_dice_reproduce}. If we look at training curves carefully, we would observe that BC could perform well in the initial learning phase. As the number of iterations increases, the performance of BC tends to degenerate. This is indeed the \emph{overfitting} phenomenon. 

To overcome the overfitting in this low-data regime, we can adopt many regularization techniques such as weight decay \citep{krogh1991weightdecay}, dropout \citep{srivastava2014dropout}, and early stopping \citep{prechelt1996earlystopping}. For simplicity, let us consider the weight decay technique, i.e., adding a squared $\ell_2$-norm penalty for the training parameters to the original training objective. The empirical result is shown in \cref{fig:bc_l2}. Surprisingly, we can find that even with small weight decay, BC improves its generalization performance and is on-pair with  ValueDice.

\begin{figure}[htbp]
    \centering
    \includegraphics[width=\linewidth]{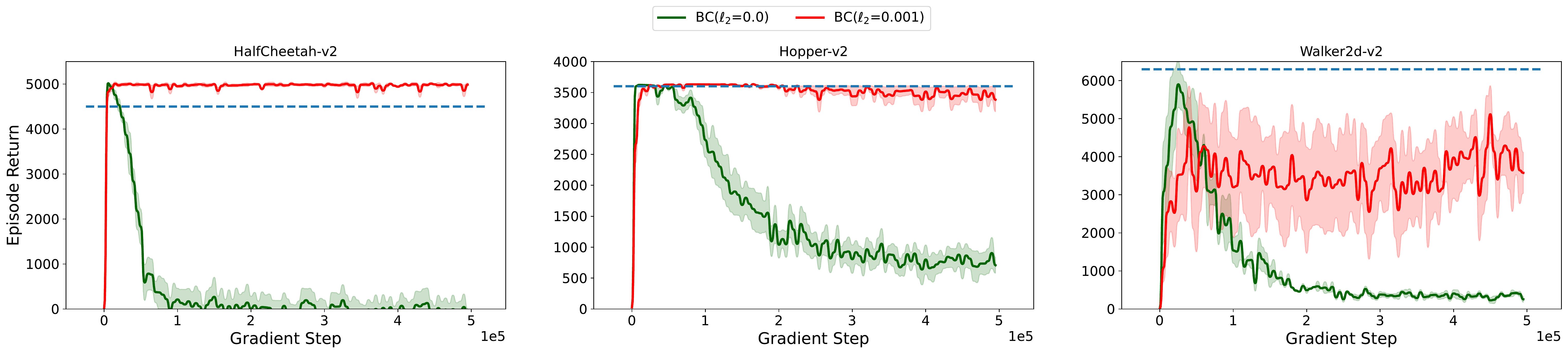}
    \caption{Comparison of BC with/without weight decay (i.e., $\ell_2$-norm based regularization) with 1 expert trajectory. Dashed lines correspond to the performance of expert policies.}
    \label{fig:bc_l2}
\end{figure}

The perspective of overfitting motivates us to carefully examine the performance of ValueDice. In particular, we realize that in practice ValueDice uses the orthogonal regularization \citep{brock2018large}, a kind of regularizer from the GAN literature. Without this regularizer, the performance of ValueDice is poor as shown in \cref{fig:value_dice_offline_variant}.

\begin{obs}
Regularization is important for offline imitation learning algorithms such as BC and ValueDice in the low-data regime.
\end{obs}

Up to now, we know that ValueDice degenerates to BC under the offline setting and the experiment results suggest that regularization matters. With proper regularization, BC and ValueDice are able to match the expert performance.

\subsection{Discussion: Why BC Performs Well?}
\label{subsec:discussion}

One related question about previous experiment results is that why BC (with regularization) performs so well (i.e., BC nearly matches the expert performance) even provided with 1 expert trajectory as in \cref{fig:bc_l2}. It seems that there are no compounding errors for BC.  Furthermore, the presented results seem to contradict the common impression that BC performs poorly but AIL performs well \citep{ho2016gail, Kostrikov19dac}. We discuss this mismatch in the following part. 

First, the compounding errors argument for BC relies on the assumption that the agent makes a mistake in each time step as in \citep{ross11dagger}. As such, decision errors accumulate over time steps. However, the assumption may not hold for some cases. Considering MDPs with deterministic transitions, the BC's policy does not make any mistake by starting from a visited initial state. As a consequence, there are no compounding errors along such a \dquote{good} trajectory.  Formally, the policy value gap of BC is expected to be $\widetilde{\gO}(|\gS| H/m)$ under MDPs with deterministic transitions, where $H$ is the finite horizon and $m$ is the number of expert trajectories; see \citep{xu2021nearly}.

\begin{thm}[BC for Deterministic MDPs \citep{xu2021nearly}]  \label{theorem:bc_deterministic}
Consider the tabular MDPs with finite horizons and deterministic transitions. Given $m$ expert trajectories with length $H$, the policy value gap of BC is upper bounded by:
\begin{align*}
    V(\pi^{\expert}) - \expect \ls V(\pi^{\operatorname{BC}}) \rs \precsim \lp H, \frac{|\gS| H}{m}\rp,
\end{align*}
where the expectation is taken over the randomness when collecting the expert dataset and $\pi^{\operatorname{BC}}$ is the \dquote{counting} solution similar to that in \eqref{eq:bc_count}.  
\end{thm}

Here the dependence on $|\gS|$ is due to the random initial states. The policy value gap becomes smaller if the initial states are limited.  We highlight that this upper bound is tighter than the general one $\widetilde{\gO}(|\gS|H^2/m)$, which are derived from MDPs with stochastic transitions \citep{rajaraman2020fundamental}. The fundamental difference is that stochastic MDPs ensure BC to have a positive probability to make a mistake even by starting from a visited initial state, which leads to the quadratic dependence on $H$. For our purpose, we know that MuJoCo locomotion tasks have (almost) deterministic transition functions and the initial states lie within a small range. Therefore, it is expected that BC (with regularization) performs well as in \cref{fig:bc_l2}.

\begin{obs}
For deterministic tasks (e.g., MuJoCo locomotion tasks), BC has no compounding errors if the provided expert trajectory are complete.
\end{obs}

Second, the worse dependence on the planning horizon for BC is observed based on \emph{subsampled} trajectories in \citep{ho2016gail, Kostrikov19dac}. More concretely, they sample non-consecutive transition tuples from a complete trajectory as the expert demonstrations. This is different from our experiment setting, in which we use \emph{complete} trajectories. It is obvious that subsampled trajectories artificially \dquote{mask} some states. As a result, BC is not ensured to match the expert trajectory even when the transition function is deterministic and the initial state is unique. This explains why BC performs poorly as in \citep{ho2016gail, Kostrikov19dac}. We note that this subsampling operation turns out minor for AIL methods (see the explanation in \citep{xu2021nearly}).

\section{Rethinking ValueDice Under the Online Setting}
\label{sec:online_setting}

In this section, we consider the online setting, where the agent is allowed to interact with the environment. In particular, it is empirically shown that ValueDice outperforms another off-policy AIL algorithm DAC \citep{Kostrikov19dac} in terms of environment interactions; see \citep[Figure 2]{Kostrikov20value_dice} or the reproduced results in the following \cref{fig:online_complete}. More surprisingly, ValueDice successfully performs policy optimization without the useful technique \emph{target network}. This contradicts the common sense that optimization of Bellman backup objectives (e.g., temporal difference learning and Q-learning) with function approximation can easily diverge (see the divergence example in \citep{baird1995residual, tsitsiklis1997analysis}. To address this divergence issue, \emph{target network} is proposed in \citep{mnih2015human} and this technique is widely applied in deep RL \citep{lillicrap2015ddpg, fujimoto2018td3, haarnoja2018sac}; see \citep{lee2019target, zhang2021breaking, agarwal2021online, chen2022target, li2022note} for the explanation about why the target network can address the divergence issue.

\begin{figure}[htbp]
    \centering
    \includegraphics[width=\linewidth]{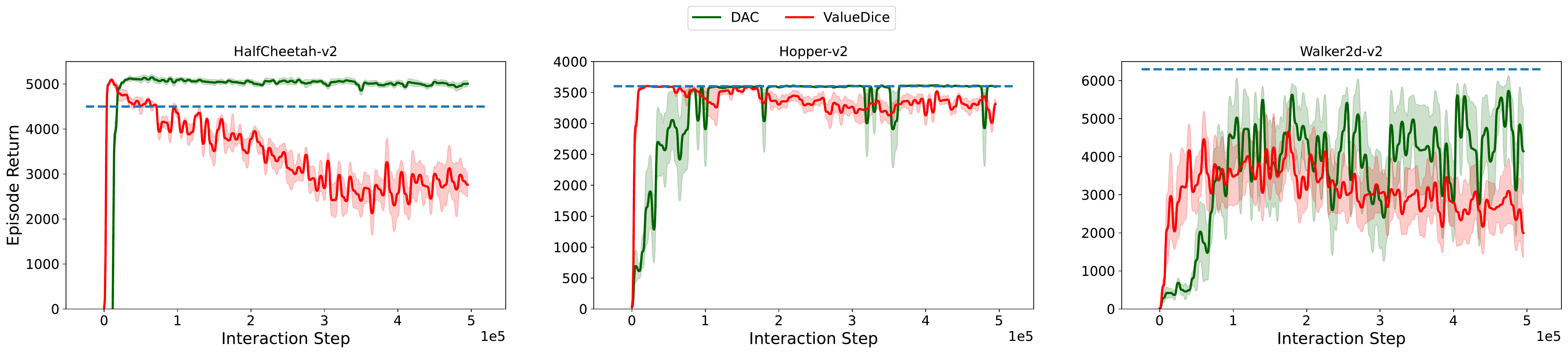}
    \caption{Comparison of ValueDice and DAC under the online setting with 1 complete expert trajectory, which reproduces \citep[Figure 2]{Kostrikov20value_dice}. Dashed lines correspond to the performance of expert policies.}
    \label{fig:online_complete}
\end{figure}

We realize that the result in \citep[Figure 2]{Kostrikov20value_dice} is based on the setting where complete expert trajectories are provided. In this case, we have known that ValueDice could degenerate to BC, and BC (with regularization) performs well. Furthermore, by comparing learning curves in \cref{fig:value_dice_reproduce} (offline) and \cref{fig:online_complete} (online), we find that the online interaction does not matter for ValueDice. This helps confirm that the reduction is crucial for ValueDice. Thus, we know the reason why ValueDice does not diverge in this case is that there is no divergence issue for BC.

Now, how about the case where expert trajectories are incomplete/subsampled? That is, the expert state-action pairs are no longer temporally consecutive, which is common in practice. The corresponding results are missing in \citep{Kostrikov20value_dice} and we provide such results in the following \cref{fig:online_subsampled}. In particular, we see that ValueDice fails but DAC still works.

\begin{figure}[htbp]
    \centering
    \includegraphics[width=\linewidth]{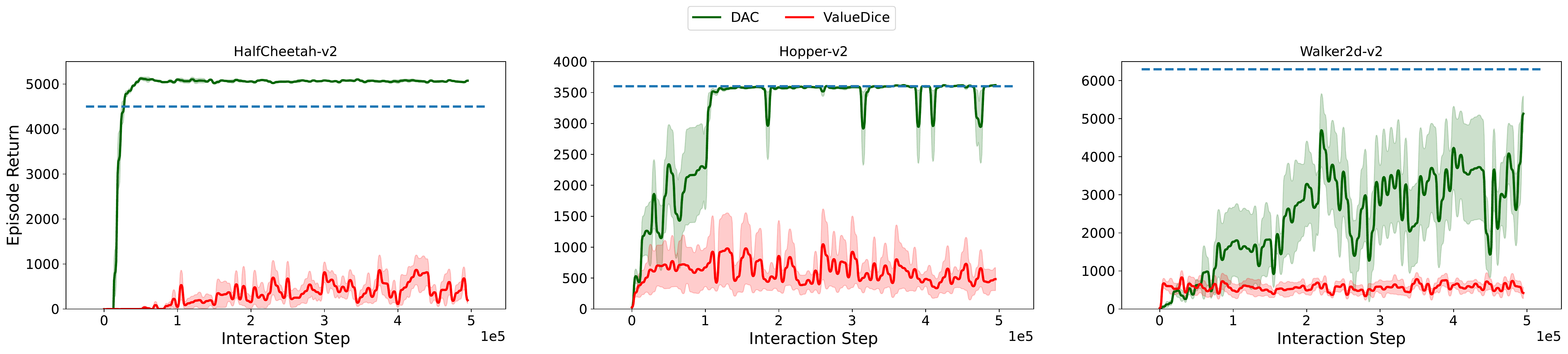}
    \caption{Comparison of ValueDice and DAC under the online setting with 10 subsampled expert trajectories (the subsampling rate is 20). Dashed lines correspond to the performance of expert policies.}
    \label{fig:online_subsampled}
\end{figure}

\begin{obs}
For MuJoCo locomotion tasks, ValueDice performs well in the complete trajectory case while it does not perform well in the subsampled case even though it can interact with the environment. 
\end{obs}

Since ValueDice and DAC use the same principle (i.e., state-action distribution matching), we believe that the poor performance of ValueDice in \cref{fig:online_subsampled} is mainly caused by optimization issues. In terms of optimization, one major difference between ValueDice and DAC is that DAC uses the target network technique while ValueDice does not. As mentioned, without a target network, optimization of Bellman backup objectives with function approximation can easily diverge. Note that the divergence issue discussed here does not contradict the good results in \cref{fig:bc_l2} and \cref{fig:online_complete} because in that case, ValueDice is closely related to BC that performs well.

In summary, our experiment results support the claim that ValueDice cannot succeed to perform policy optimization with function approximation under the subsampled cases. Moreover, our results suggest that the mentioned success of ValueDice may rely on its connection to BC.

\section{Conclusion}
\label{sec:conclusion}

In this manuscript, we rethink the algorithm designs in ValueDice and promote new conclusions under both offline and online settings. We clarify that our results do not indicate that ValueDice is \dquote{weaker} than BC or ValueDice does not show any algorithmic insights. Instead, our studies highlight the connection between adversarial imitation learning algorithms (including ValueDice) and BC under the offline setting, point out the instability of Dice-based technique under the certain scenarios, and clarify some confusing results in the literature. We notice that the ideas from ValueDICE are still valuable, which can be extended to other imitation learning and offline reinforcement learning algorithms. Perhaps, some conclusions in this manuscript could provide insights to examine the algorithmic advances and understand the reported results in the related works such as SoftDice \citep{sun2021softdice}, OptDice \citep{lee2021optdice}, SmoDice \citep{ma2022smodice}, and DemoDice \citep{kim2021demodice}. 

For the general imitation learning studies, our work has the following implications.
\begin{itemize}
    \item \textbf{Algorithm Evaluations.} Our experiment results show a clear boundary between the complete trajectory and subsampled trajectory cases. Two cases have dramatically different characteristics, results and explanations. As a result, we must be cautious to evaluate algorithms by identifying the context.  Without this, some arguments may be misleading for future studies. 
    \item \textbf{Benchmarks.} Our study points out several drawbacks of existing MuJoCo locomotion benchmarks: deterministic transitions and limited initial states. For one thing, provided 1 complete expert trajectory, simple algorithm BC is competent for many tasks, which is not discovered in the previous research. That is, current tasks are not even \dquote{hard}. For another thing, in addition to large-scale studies in \citep{hussenot2021hyperparameter, orsini2021what} that mainly focus on existing MuJoCo locomotion benchmarks, future imitation learning studies could benefit from more challenging benchmarks with stochastic transitions and diverse initial states\footnote{{We notice that \citet{spencer2021feedback} also considered the limitation of MuJoCo locomotion benchmarks. Specifically, they argued that these benchmarks are realizable (i.e., no approximation error) and thus are not sufficient to cover some hard regimes in real-world  imitation learning applications.} }.
\end{itemize}

\section*{Acknowledgements and Disclosure of Funding}

We thank anonymous reviewers {and Gokul Swamy} for the helpful comments. The work of Yang Yu is supported by National Key R\&D Program of China National Key Research and Development Program of China (2020AAA0107200), NSFC(61876077), and Collaborative Innovation Center of Novel Software Technology and Industrialization. The work of Zhi-Quan Luo is supported by the National Natural Science Foundation of China (No. 61731018) and the Guangdong Provincial Key Laboratory of Big Data Computation Theories and Methods.

\bibliographystyle{abbrvnat}
\bibliography{reference}  

\clearpage
\newpage
\appendix 

\section{Missing Proofs}

\newcommand{\neigh}{\operatorname{neigh}}

\subsection{Episodic Markov Decision Process}

In this part, we briefly introduce the episodic Markov decision process. We remark that episodic MDP is introduced to facilitate analysis since the horizon is finite. 

An episodic Markov decision process (MDP) can be described by the tuple $\gM = (\gS, \gA, \gP, r, p_0, H)$. Here $\gS$ and $\gA$ are the state and action space, respectively. $H$ is the planning horizon\footnote{$H$ corresponds to the effective horizon $1/(1-\gamma)$ in the infinite-horizon discounted MDP.}, $p_0$ is the initial state distribution, and $r: \gS \times \gA \rar [0, 1]$ is the reward function. Different from the infinite-horizon discounted MDP in \cref{sec:background}, transition function $\gP = \{p_1, \cdots, p_{H}\}$ is non-stationary. That is, $p_h(s_{h+1}|s_h, a_h)$ determines the probability of transiting to state $s_{h+1}$ conditioned on state $s_h$ and action $a_h$ in time step $h$, for $h \in [H]$\footnote{$[x]$ denotes the set of integers from $1$ to $x$.}. Since the optimal policy is not stationary, we need to consider the  non-stationary policy $\pi = \lb \pi_1, \cdots, \pi_h \rb$, where $\pi_h: \gS \rar \Delta(\gA)$ and $\Delta(\gA)$ is the probability simplex, $\pi_h (a|s)$ gives the probability of selecting action $a$ on state $s$ in time step $h$, for $h \in [H]$. 

The policy value is defined as the cumulative rewards over $H$ time steps:
\begin{align*}
    V(\pi) := \expect \ls \sum_{t=1}^{H} r(s_t, a_t) \mid s_0 \sim p_0(\cdot), a_t \sim \pi_t(\cdot|s_t), s_{t+1} \sim p_t(\cdot|s_t, a_t)  \rs .
\end{align*}
Again, we can define the non-stationary state-action distribution $\dpi_t$:
\begin{align*}
    \dpi_t(s, a) := \sP(s_t = s, a_t=a),
\end{align*}
which characterizes the probability of visiting the state-action pair $(s, a)$ in time step $t$.

\subsection{Reduction from AIL to BC}
\label{appendix:reduction_from_ail_to_bc}

\begin{thm} \label{theorem:episodic_mdp}
For episodic MDP, consider the following optimization problem:
\begin{align*}
   \min_{\pi} \sum_{t=1}^{H} \sum_{(s, a) \in \gS \times \gA} \, \, &\labs \dpi_t(s, a) - \widehat{d^{\expert}_t}(s, a) \rabs, \\
    \text{such that} \,\, & \pi_t(a|s) = 1/|\gA|, \forall s, t \text{ with } \widehat{d^{\expert}_t}(s) = 0,
\end{align*}
Assume that there is only one expert trajectory. Then, we have that the BC's policy is the globally optimal solution.
\end{thm}

\begin{proof}
To facilitate later analysis, we will use $a^{\expert}$ to denote the expert action on a specific state $s$ (recall that we assume the expert policy is deterministic). Without loss of generality, let $\widetilde{s}$ be visited state occurred in the single expert trajectory.

We want to mimic the classical dynamic programming proof as in \citep{bertsekas2012dynamic}. Specifically, we aim to prove that the BC's policy is optimal in a backward way. That is, we will use the induction proof. To this end, we define the cost function in time step $h$ as 
\begin{align*}
   \text{Loss}_h =  \sum_{(s, a) \in \gS \times \gA} \labs \dpi_h(s, a) - \widehat{d^{\expert}_h}(s, a) \rabs.
\end{align*}
Then, we can define the cost-to-go function:
\begin{align*}
    \ell_{h} = \sum_{t=h}^{H} \sum_{(s, a) \in \gS \times \gA} \labs \dpi_t(s, a) - \widehat{d^{\expert}_t}(s, a) \rabs.
\end{align*}
For the base case where time step $h = H$, we have that 
\begin{align*}
    \ell_{H} &= \sum_{(s, a)} \labs \dpi_H(s, a) - \widehat{d^{\expert}_H}(s, a) \rabs \\
    &= \sum_{a} \labs d^{\pi}_{H}(\widetilde{s}, a) - \widehat{d^{\expert}_H}(\widetilde{s}, a) \rabs + \sum_{s \ne \widetilde{s}} \sum_{a} \labs d^{\pi}_H(s, a) -  \widehat{d^{\expert}_H}(s, a) \rabs  \\
    &=  \labs d^{\pi}_{H}(\widetilde{s}, a^{\expert}) - 1 \rabs + \sum_{a \ne a^{\expert}}  \labs d^{\pi}(\widetilde{s}, a^{\expert}) - 0 \rabs  + \sum_{s \ne \widetilde{s}} \sum_{a} \labs d^{\pi}_H(s, a) - 0 \rabs \\
    &= \labs d^{\pi}_H(\widetilde{s}) \pi_{H}(a^{\expert}| \widetilde{s}) - 1 \rabs + d^{\pi}_H(\widetilde{s}) ( 1 - \pi_H(a^{\expert} | \widetilde{s})) + \sum_{s \ne \widetilde{s}} \sum_{a} d^{\pi}_H(s) \frac{1}{|\gA|} \\
    &= 1 + d^{\pi}_H(\widetilde{s}) - 2 d^{\pi}_H(\widetilde{s}) \pi_{H}(a^{\expert}|\widetilde{s})  + \sum_{s \ne \widetilde{s}} d^{\pi}_H(s) \\
    &= 1 + d^{\pi}_H(\widetilde{s}) - 2 d^{\pi}_H(\widetilde{s}) \pi_{H}(a^{\expert}|\widetilde{s})  + (1 - d^{\pi}_H(\widetilde{s})) \\
    &= 2 \lp 1 - d^{\pi}_H(\widetilde{s}) \pi_H(a^{\expert} | \widetilde{s}) \rp,
\end{align*}
which has a globally optimal solution at $\pi_H(a^{\expert}| \widetilde{s}) = 1$. This proves the base case. 

Next, for the induction step, assume we have that 
\begin{align*}
    (\pi_{h+1}(a^{\expert} | \widetilde{s}) = 1,  \pi_{h+2}(a^{\expert}|\widetilde{s}) = 1, \cdots,  \pi_{H}(a^{\expert}|\widetilde{s}) = 1) = \argmin_{\pi_{h+1}, \pi_{h+2}, \cdots, \pi_{H} } \ell_{h+1}(\pi_{h+1}, \pi_{h+2}, \cdots, \pi_{H}),
\end{align*}
where the left hand side is the \emph{unique} globally optimal solution. We want to prove that \begin{align*}
    (\pi_{h}(a^{\expert}|\widetilde{s}) = 1,  \pi_{h+1}(a^{\expert} | \widetilde{s}) = 1,   \cdots,  \pi_{H}(a^{\expert}|\widetilde{s}) = 1) = \argmin_{\pi_h, \pi_{h+1},  \cdots, \pi_{H} } \ell_{h}(\pi_h, \pi_{h+1}, \cdots, \pi_{H}) 
\end{align*}
where the left hand side is the \emph{unique} globally optimal solution.

For time step $h$, we have that 
\begin{align*}
    \text{Loss}_h  &= \sum_{(s, a)} \labs \dpi_h(s, a) - \widehat{d^{\expert}_h}(s, a) \rabs \\
    &= 2( 1 - d^{\pi}_h(\widetilde{s}) \pi_h(a^{\expert}|\widetilde{s})).
\end{align*}
For time step $h+ 1 \leq h^\prime \leq H$, we have that 
\begin{align*}
    \text{Loss}_{h^\prime}  &= \sum_{(s, a)} \labs \dpi_{h^\prime}(s, a) - \widehat{d^{\expert}_{h^\prime}}(s, a) \rabs \\
    &= 2( 1 - d^{\pi}_{h^\prime}(\widetilde{s}) \pi_{h^\prime}(a^{\expert} | \widetilde{s})) \\
    &= 2( 1 - d^{\pi}_{h^\prime}(\widetilde{s})),
\end{align*}
where the last step follows our assumption. We have the following decomposition for $d^{\pi}_{h^\prime}(\widetilde{s})$, 
\begin{align*}
    d^{\pi}_{h^\prime}(\widetilde{s}) &= \sum_{(s, a)} p_{h^\prime - 1}(\widetilde{s}|s, a) d_{h^\prime -1}^{\pi}(s, a) \\
    &= p_{h^\prime - 1}(\widetilde{s}|\widetilde{s}, a^{\expert}) d^{\pi}_{h^\prime - 1}(\widetilde{s}, a^{\expert}) +  \sum_{a \ne a^{\expert}} p_{h^\prime - 1}(\widetilde{s} | \widetilde{s}, a) d_{h^\prime - 1}^{\pi}(\widetilde{s}, a) + \sum_{s \ne \widetilde{s}} \sum_{a} p_{h^\prime - 1}(\widetilde{s}|s, a) d^{\pi}_{h^\prime -1} (s, a) \\
    &= p_{h^\prime -1}(\widetilde{s} | \widetilde{s}, a^{\expert}) d^{\pi}_{h^\prime -1}(\widetilde{s})  \pi_{h^\prime-1}(a^{\expert} | \widetilde{s}) + \sum_{a \ne a^{\expert}}p_{h^\prime -1}(\widetilde{s} | \widetilde{s}, a) d_{h^\prime -1}(\widetilde{s}) \pi_{h^\prime -1}(a| \widetilde{s}) \\
    &\quad + \sum_{s \ne \widetilde{s}} \sum_{a} p_{h^\prime - 1}(\widetilde{s}|s, a) d^{\pi}_{h^\prime -1} (s) \frac{1}{|\gA|} \\
    &= p_{h^\prime -1 } (\widetilde{s}|\widetilde{s}, a^{\expert}) d^{\pi}_{h^\prime -1}(\widetilde{s}) + \sum_{s \ne \widetilde{s}} \sum_{a} p_{h^\prime - 1}(\widetilde{s}|s, a) d^{\pi}_{h^\prime -1} (s) \frac{1}{|\gA|},
\end{align*}
where the last equation again follows our assumption.

Finally, we should have that\footnote{We have four types of transition flows: [1] $(\widetilde{s}, a^{\expert}) \rar (\widetilde{s}, a^{\expert}) \rar \cdot \rar (\widetilde{s}, a^{\expert}) \rar \widetilde{s}$; [2] $(\widetilde{s}, a) \rar (\widetilde{s}, a^{\expert}) \rar \cdots \rar (\widetilde{s}, a^{\expert}) \rar \widetilde{s}$; [3] $(\widetilde{s}, a^{\expert}) \rar (s, a) \rar \cdots \rar (s, a) \rar \widetilde{s}$; [4] $(\widetilde{s}, a) \rar (s, a) \rar \cdots \rar (s, a) \rar \widetilde{s}$. The first two terms arise in $(a)$ and the last two terms arise in $(b)$. } $\ell_h(\pi_h) = (a) + (b) + \text{constant} $, where $\text{constant}$ is unrelated to $\pi_h(\cdot| \widetilde{s})$ and 
\begin{align*}
    &(a) = -\prod_{\ell = h + 1}^{h^\prime -1 } p_{\ell}(\widetilde{s}|\widetilde{s}, a^{\expert})  \ls p_{h}(\widetilde{s}|\widetilde{s}, a^{\expert})  \pi_h(a^{\expert} | \widetilde{s}) d^{\pi}_h(\widetilde{s}) + \sum_{a \ne a^{\expert}} p_{h}(\widetilde{s}| \widetilde{s}, a) \pi_{h}(a| \widetilde{s}) d^{\pi}_h(\widetilde{s}) \rs - 2d^{\pi}_h(\widetilde{s}) \pi_h(a^{\expert}|\widetilde{s}), \\
    &(b) =  -\sum_{s \ne \widetilde{s}} \sum_{a} \frac{1}{|\gA|}  p_{h^\prime - 1}(\widetilde{s}| s, a) \bigg[ \sP(s_{h^\prime - 1} = s|s_h = \widetilde{s}, a_h = a^{\expert}) \pi_{h} (a^{\expert} |\widetilde{s}) d^{\pi}_h(\widetilde{s}) \\
    &\qquad\qquad\qquad\qquad\qquad\qquad + \sum_{a \ne a^{\expert }} \sP(s_{h^\prime - 1} = s | s_h = \widetilde{s}, a_h = a) \pi_h(a|\widetilde{s}) d^{\pi}_h(\widetilde{s}) \bigg].
\end{align*}
Let us compare the coefficient for $\pi_h(a^{\expert} | \widetilde{s})$ and $\pi_h(a|\widetilde{s})$ with $a \ne a^{\expert}$:
\begin{align*}
    c^{\expert} &= -d^{\pi}_h(\widetilde{s}) \lb \underbrace{\prod_{\ell = h + 1}^{h^\prime -1 } p_{\ell }(\widetilde{s}|\widetilde{s}, a^{\expert}) p_{h}(\widetilde{s}|\widetilde{s}, a^{\expert})}_{ \in (0, 1]} + \underbrace{\sum_{s \ne \widetilde{s}} \sum_{a} \frac{1}{|\gA|} p_{h^\prime - 1}(\widetilde{s}| s, a)  \sP(s_{h^\prime - 1} = s|s_h = \widetilde{s}, a_h = a^{\expert})}_{\in [0, 1)} + 2 \rb, \\
    c^{a} &= -d^{\pi}_h(\widetilde{s}) \lb  \underbrace{\prod_{\ell = h + 1}^{h^\prime -1 } p_{\ell}(\widetilde{s}|\widetilde{s}, a^{\expert}) p_{h}(\widetilde{s}|\widetilde{s}, a)}_{\in [0, 1]} + \underbrace{\sum_{s \ne \widetilde{s}} \sum_{a} \frac{1}{|\gA|}  p_{h^\prime - 1}(\widetilde{s}| s, a) \sP(s_{h^\prime - 1} = s | s_h = \widetilde{s}, a_h = a)}_{\in [0, 1]} \rb.
\end{align*}
Hence, $c^{\expert} < c^{a}$ and we know that $a^{\expert}$ is the optimal action in time step $h$. This proves the induction case. 

\end{proof}

\newpage
\section{Experiment Details}
\label{appendix:experiment_details}

\textbf{Algorithm Implementation.} Our implementation of ValueDice and DAC follows the public repository \url{https://github.com/google-research/google-research/tree/master/value_dice} by \citet{Kostrikov20value_dice}. Our implementation of BC is different from the one in this repository. In particular, \citet{Kostrikov20value_dice} implemented BC with a Gaussian policy with trainable mean and covariance as in \citep[Figure 4]{Kostrikov20value_dice}. However, we observe that the performance of this implementation is very poor because the mean and covariance share the same hidden layers and the covariance affects the log-likelihood estimation. Instead, we use a simple MLP architecture without the output of the covariance. This deterministic policy is trained with mean-square-error (MSE):
\begin{align*}
   \min_{\theta} \sum_{(s, a) \sim \gD} (f_{\theta}(s) - a)^2.
\end{align*}
The hidden layer size and optimizer of our BC policy follow the configuration for ValueDice. 

\textbf{Benchmarks.} All preprocessing procedures follow \citep{Kostrikov20value_dice}. The subsampling procedure follows \citep{Kostrikov19dac}; please refer to \url{https://github.com/google-research/google-research/blob/master/dac/replay_buffer.py#L154-L177}. The expert demonstrations are from \url{https://github.com/google-research/google-research/tree/master/value_dice#install-dependencies}.

\textbf{Experiments.} All algorithms are run with 5 random seeds (2021-2025). For all plots, solid lines correspond to the mean, and shaded regions correspond to the standard deviation.

\end{document}